%% file: main.tex
\pdfoutput=1

\documentclass[11pt]{article}

\usepackage[]{acl_2024}
\usepackage{custom}
\input{macros}

\title{An $\mathbf{\lstar}$ Algorithm for Deterministic Weighted Regular Languages}

\author{%
Clemente Pasti
~\;~\;~Talu Karag{\"o}z
~\;~\; Anej Svete \\
\textbf{Franz Nowak} \;~\textbf{Reda Boumasmoud}\;~\textbf{Ryan Cotterell}\\
\{\texttt{\href{mailto:clemente.pasti@inf.ethz.ch}{clemente.pasti}}, \texttt{\href{mailto:anej.svete@inf.ethz.ch}{anej.svete}}, \texttt{\href{mailto:franz.nowak@inf.ethz.ch}{franz.nowak}}, \texttt{\href{mailto:ryan.cotterell@inf.ethz.ch}{ryan.cotterell}}\}\texttt{@inf.ethz.ch} \\
\texttt{\href{mailto:reda.boumasmoud@math.ethz.ch}{reda.boumasmoud@math.ethz.ch}},\: \texttt{\href{mailto:talukaragoz@gmail.com}{talukaragoz@gmail.com}} \\
\includegraphics[width=2.8cm]{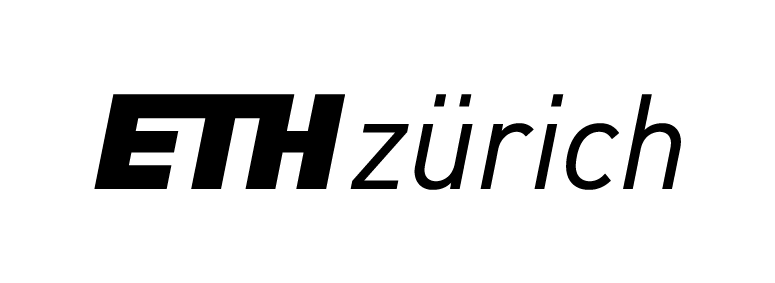}
 }
\begin{document}
\maketitle
\begin{abstract}
    Extracting finite state automata (FSAs) from black-box models offers a powerful approach to gaining interpretable insights into complex model behaviors.
    To support this pursuit, we present a weighted variant of \citeposs{ANGLUIN198787} \lstar algorithm for learning 
    FSAs. 
    We stay faithful to the original algorithm, devising a way to exactly learn deterministic weighted FSAs whose weights support division.
    Furthermore, we formulate the learning process in a manner that highlights the connection with FSA minimization, showing how \lstar directly learns a minimal automaton for the target language.
    
    \vspace{1em}
    {\includegraphics[width=1.36em,height=1.25em]{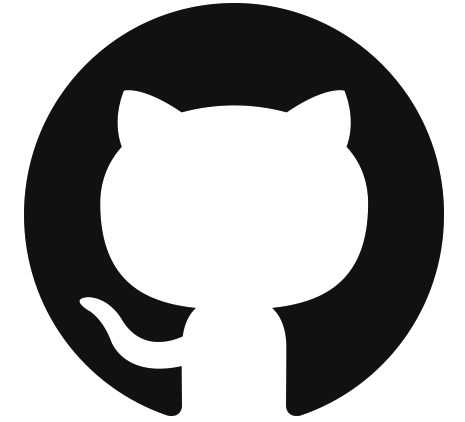}\hspace{2pt}\raisebox{0.5\height}{\parbox{\dimexpr\linewidth-2\fboxsep-2\fboxrule}{\href{https://github.com/rycolab/weighted-angluin}{\texttt{github.com/rycolab/weighted-angluin}}}}}
\end{abstract}

\section{Introduction}
Learning formal languages from data is a classic problem in computer science.
Unfortunately, learning only from positive examples is impossible \citep{GOLD}.
By granting the learner access to more than just positive examples, \citet{ANGLUIN198787} introduced the \emph{active} learning scheme \lstar, where the learner interacts with an oracle by asking it queries.
Concretely, \citeposs{ANGLUIN198787} \lstar algorithm learns regular languages in the form of deterministic finite-state automata (DFSAs) from \emph{membership} queries (analogous to asking for a ground truth label of a string in the training dataset) and \emph{equivalence} queries (analogous to asking whether a hypothesis is correct).\looseness=-1

Weighted formal languages, where strings are assigned weights such as probabilities or costs, naturally generalize membership-based (boolean) formal languages.
Weighted languages, especially probabilistic languages, serve as a cornerstone in the conceptual framework of many NLP problems \citep{mohri-1997-finite}. 
Their significance is twofold: First, in practical applications, where they underpin algorithms for tasks such as parsing \citep{goodman-1996-parsing} and machine translation \citep{mohri-1997-finite}, and second, as an analytical framework for better understanding modern language models \citep[\emph{inter alia}]{pmlr-v80-weiss18a,jumelet-zuidema-2023-transparency,nowak-etal-2024-computational}.
This has motivated the development of various weighted extensions of \citeposs{ANGLUIN198787} \lstar algorithm.
For instance, \citet{weiss-weighted-l} describes a generalization that (approximately) learns a probabilistic DFSA by querying a neural language model to interpret it. 
Less faithfully to the original \lstar algorithm,  multiple algorithms for learning \emph{non-deterministic} weighted FSAs have been proposed 
\citep{Learning-behaviours-bergadano,beimel,Balle-Mohri-Spectral,Balle-Quattoni-2014,daviaud2024feasability,Balle2015}. 
These algorithms involve the solution of a linear system of equations, and therefore they cannot be used when the underlying algebraic structure lacks subtraction.

We present a novel weighted generalization 
of the \lstar algorithm that learns \emph{semifield}-
weighted \emph{deterministic} FSAs.
In contrast to other algorithms inspired by \lstar, ours is a more faithful generalization of the original learning scheme. We 
generalize  \citeposs{ANGLUIN198787} original algorithm, resulting in a familiar procedure that, 
just like the original, learns a DFSA exactly in a finite number of steps if the automaton can be determinized.\footnote{All 
boolean-weighted FSA can be determinized, which is why \citeposs{ANGLUIN198787} \lstar always halts.}
Additionally, we loosen the requirement for field-weighted 
FSAs; our algorithm works for semifield-weighted FSA.
Our exposition further illuminates the 
connection between weighted minimization 
\citep{Hopcroft-Automata-Theory,mohri-1997-finite} and \lstar.\looseness=-1
\section{Weighted Regular Languages}
\paragraph{Semirings and Semifields.}
Throughout this paper, we fix a \textbf{semifield} $ K = (\weightset, \oplus, \otimes, \zero, \one) $, where $ \weightset $ is a set equipped with two associative laws, $ \oplus $ and $ \otimes $, along with distinguished elements $ \zero $ and $ \one $, satisfying the following conditions:
\begin{enumerate}[left=0pt,labelwidth=0pt, labelsep=0pt,itemsep=-0em]
    \item $ (\weightset, \oplus, \zero) $ is a \defn{commutative monoid}, %
    \item $ (\weightset \setminus \{\zero\}, \otimes, \one) $ is a \defn{group}, %
    \item The law $\otimes $ distributes over $ \oplus $, so for all $ \weight_1, \weight_2, \weight_3 \in \weightset $,
    $ (\weight_1 \oplus \weight_2) \otimes \weight_3 = (\weight_1 \otimes \weight_3) \oplus (\weight_2 \otimes \weight_3) $ and $ \weight_1 \otimes (\weight_2 \oplus \weight_3) = (\weight_1 \otimes \weight_2) \oplus (\weight_1 \otimes \weight_3) $; and
\item $ \zero $ acts as an annihilator for $ \otimes $, meaning $ \weight \otimes \zero = \zero \otimes \weight = \zero $ for all $ \weight \in \weightset $.
\end{enumerate} 

\paragraph{Strings and Languages.}
An \textbf{alphabet} $\alphabet$ is a non-empty, finite set of symbols.
A \defn{string} is a finite sequence of symbols from an alphabet. We write $\strx \str$ to denote the concatenation of the strings $\strx$ and $\str$. Let $\alphabet^{n+1} \defeq \{ \str \sym  \mid \str \in \alphabet^{n}, \sym \in \alphabet \}$ and $\alphabet^0 \defeq \{\eps\}$, where $\varepsilon$ is the empty string. The \defn{Kleene closure} $\kleene{\alphabet} \defeq \bigcup_{n=0}^{\infty} \alphabet^{n}$ of $\alphabet$ is the set containing all strings made with symbols of $\alphabet$. We further introduce the set $\alphabet^{\leq k}= \bigcup_{n=0}^{k} \alphabet^{n}$.
Given an alphabet $\alphabet$ and a semiring $\semiringtuple$, a \textbf{weighted formal language} is a function $\lang\colon \kleene{\alphabet} \to \weightset$ that assigns weights $\weight \in \weightset$ to strings $\str \in \kleene{\alphabet}$. Unless differently specified, in this paper we will assume that all weighted languages are \emph{semifield}-weighted.
\looseness=-1
\paragraph{Weighted Finite-state Automata.}
A \textbf{weighted finite-state automaton} (WFSA) $\automaton$ over a semifield $\semiringtuple$ is a 5-tuple $\wfsatuple$ where $\alphabet$ is an alphabet, $\states$ is a finite set of states, $\transf $ is a set of weighted arcs rendered as $\edge{p}{a}{\weight}{q}$ with $p,q \in \states$, $a \in \alphabet$, and $\weight \in \weightset$,\footnote{We do not consider $\eps$-transitions. 
This is without loss of generality; any regular language can be represented by an $\eps$-free automaton \citep[Theorem 7.1]{Mohri2009}.} and $\initf\colon \states \to \weightset$ and $\finalf\colon \states \to \weightset$ are the initial and final weight function, respectively.
A \defn{path} $\apath$ in $\automaton$ is a finite sequence of contiguous arcs, denoted as
\begin{equation}
    \edge{q_0}{a_1}{\weight_1}{q_1},\cdots,\edge{q_{N-1}}{a_N}{\weight_N}{q_N}.
\end{equation}
We call $\firstq{\apath}= q_0$ the initial state of the path, and $\lastq{\apath} = q_N$ the final state of the path.
The \defn{weight} of $\apath$ is 
$\pathweight{\apath}=  \weight_1 \otimes \cdots 
\otimes \weight_N $ and its \defn{yield} is 
$\pathyield{\apath} = a_1 \cdots a_N$. With 
$\Paths{\automaton}$, we denote the set of all 
paths in $\automaton$, and with 
$\Paths{\automaton}(\prefix)$ the subset of all 
paths in $\automaton$ with yield $\prefix$.
A state $\stateq$ is \defn{accessible} if there exists a path $\apath$ with $\weight(\apath) \neq \zero$, $\initf(\firstq{\apath})\neq \zero$ and $\lastq{\apath}=\stateq$.
It is \defn{coaccessible} if there exists a path $\apath$ with $\weight(\apath) \neq \zero$, $\finalf(\lastq{\apath})\neq \zero$ and $\firstq{\apath}=\stateq$. A WFSA is \defn{trimmed} if all its states are simultaneously accessible and coaccessible.
We say that a WFSA $\automaton = \wfsatuple$ is \textbf{deterministic} (a WDFSA) if, for every $\statep \in \states, \sym \in \alphabet$, there is at most one $\stateq \in \states$ such that $\edge{\statep}{a}{\weight}{\stateq} \in \transf$ with $\weight > 0$, and there is a single state $\qinit$ with $\initf\left(\qinit\right) \neq 0$. In such case, we refer to $\qinit$ as the \defn{initial state}.
A WDFSA can have at most one path yielding a string $\str \in \kleene{\alphabet}$ from the initial state $\qinit$. 
A WDFSA $\automaton$ is said to be \defn{minimal} if no equivalent WDFSA with fewer states exists. 
\looseness=-1
\paragraph{Weighted Regular Languages.}
Every WFSA $\automaton$ generates the weighted language\looseness=-1
\begin{equation}
    \AutLang{\automaton}(\prefix) \defeq\!\! 
    \bigoplus_{ \apath \in \Paths{\automaton}
    {\prefix} } \!\! \initf(\firstq{\apath}) 
    \otimes \pathweight{\apath} \otimes 
    \finalf(\lastq{\apath})
\end{equation}
for $\prefix \in \kleene{\alphabet}$.
We define the set $\supp{\lang}= \{ \prefix \in \kleene{\alphabet} \mid \lang(\prefix) \neq \zero\}$ to be the \defn{support} of $\lang$.
A weighted language is said to be \defn{regular} 
if there exists a WFSA that generates it. If two 
WFSAs generate the same language, they are said 
to be \defn{equivalent}.
Finally, a weighted regular language is said to 
be \defn{deterministic} if there exists a WDFSA 
that generates it. 
In contrast to the boolean case, not every 
weighted regular language can be generated by a 
deterministic WFSA \citep{twins}.
Weighted deterministic regular languages are 
thus a strict subset of weighted regular 
languages.
This distinction plays a critical role in our 
exposition---we develop a generalization of 
\citeposs{ANGLUIN198787} algorithm that learns 
weighted \emph{deterministic} regular 
languages.\looseness=-1

\paragraph{Homothetic equivalence.}
Let $X$ be a subset of $\kleene{\alphabet}$ and 
$\semiringtuple$ a semifield. Let us denote with 
$\Languages{X}$ the set of functions 
$X^\weightset$. We introduce the equivalence 
relation 
\begin{align}
    \lang_1 \equiv_{X} \lang_2 &\iff \exists k\in 
    \weightset\setminus\{\zero\}:  
 \\ &    \lang_1(\strx) = k \otimes \lang_2(\strx), \; \forall \strx \in X \nonumber
 \label{eqn:homothetic}
 \end{align}
between any two functions $L_1, L_2 $ in $\Languages{X}$. 
We call this relation \defn{homothetic equivalence}. 

For every string $\strx \in \kleene{\alphabet}$, we introduce the \defn{right language} $\rightlang{\lang}{\strx}\colon \str \mapsto \lang(\strx \str)$, and define the following equivalence relation on $\kleene{\alphabet}$:
\begin{align}
    \strx \sim_{\lang} \strz \iff  \rightlang{\lang}{\strx} \equiv_{\kleene{\alphabet}} \rightlang{\lang}{\strz}
\end{align}\looseness=-1

\looseness=-1

\section{Empirical Hankel Systems}

\paragraph{Hankel matrices.}
Let $\prefixes \subseteq \kleene{\alphabet}$ be a prefix-closed set of prefixes and let $\suffixes \subseteq \kleene{\alphabet}$ be a suffix-closed set of suffixes.\footnote{A set of strings is prefix-closed (suffix-closed) if it contains all prefixes (suffixes) of each of its elements. In particular $\eps \in \prefixes \cap \suffixes$.} 
An \defn{empirical Hankel matrix} is a map $\hankel \colon {\prefixesplus\times\suffixesplus}\to \weightset$. 
For every $\prefix \in \prefixesplus$ we define the right language map $\hankel_{\prefix}\colon \suffixesplus \rightarrow \weightset, \suffix \mapsto \hankel(\prefix, \suffix)$.
Using homothetic equivalence (\cref{eqn:homothetic}), we introduce the equivalence relation $\sim_\hankel$ on $\prefixesplus$ for $\statep, \stateq \in \prefixesplus$ as\looseness=-1
\begin{equation}
    \statep \sim_\hankel \stateq  \iff {\hankel_{\statep}}   \equiv_{\suffixesplus} \hankel_{\stateq }
\end{equation} 
We denote $\prefix$'s equivalence class by $[\prefix]= \{ \stateq  \in \prefixesplus \mid  \stateq \sim_{\hankel} \prefix\}$.\looseness=-1

\paragraph{The na{\"i}ve Hankel automaton.}Given an empirical Hankel matrix $\hankel$, consider the map 
\begin{equation}
    d_\hankel \colon \prefixesplus \to \weightset,\, \prefix \mapsto \oplus_{\suffix \in \suffixes} \hankel(\prefix, \suffix)
\end{equation}
We introduce it here to streamline the construction of \defn{the na{\"i}ve Hankel automaton} associated with $\hankel$; the WFSA\footnote{This automaton is \emph{not} necessarily determinisitc.} $\naivehankelautomaton=\wfsatuple{_\hankel}$ with:
\begin{enumerate}[label=\textit{(\arabic*)}, nosep,wide, labelindent=0pt]
    \item \textbf{States.} 
 We define the states $\states_\hankel \defeq\prefixes$. 
    \item \textbf{Transitions.}
For every state $\statep\in \prefixes$ and every symbol $\sym \in \alphabet$, let the transition $\edge{{\statep}}{\textcolor{ETHRed}{\sym}}{\weight}{   {\statep '}   }$ be in $\transf_\hankel$ whenever $\statep \,\textcolor{ETHRed}{\sym} \sim_\hankel \statep'$ and where 
\vspace{-2mm}
          \begin{equation}
              \weight \defeq \begin{cases}
                  \frac{ d_\hankel (\statep\textcolor{ETHRed}{\sym}) }{d_\hankel(\statep)} & \ifcondition d_\hankel(\statep)\neq 0,
              \\
              \zero & \otherwisecondition
              \end{cases}
          \end{equation}
    \item \textbf{Initial weight.}
    For every state $ \statep \in \prefixes$, we define its initial weight as
          \vspace{-2 mm}
          \begin{equation}
              \initf_\hankel(\statep) \defeq \begin{cases}
                d_\hankel(\eps) & \ifcondition \statep = \eps,    \\
             \zero & \otherwisecondition
  \end{cases}
              \label{eqn: hypothesis automaton initial weights}
          \end{equation}
    \item \textbf{Final weights.}
    For every state $\statep \in \prefixes$,
          we define its final weight
          \vspace{-2 mm}
          \begin{equation}
              \finalf_\hankel(\statep) \defeq \begin{cases}
              \frac{\hankel(\statep,\eps)}{ d_\hankel(\statep)} & \ifcondition d_\hankel(\statep) \neq \zero,    \\
             \zero                                                      & \otherwisecondition
  \end{cases}
              \label{eq:final-weight-hypothesis}
          \end{equation}
\end{enumerate}\looseness=-1

\paragraph{Empirical Hankel systems.} 
\begin{definition}
     An \defn{empirical Hankel system} is a triplet $\emphankel = (\prefixes, \suffixes, \hankel)$, where $\prefixes \subseteq \kleene{\alphabet}$ prefix closed, $\suffixes \subseteq \kleene{\alphabet}$ a suffix closed, and $\hankel\colon \prefixesplus \times \suffixesplus \to \weightset$ is an empirical Hankel matrix that is:
     \begin{enumerate}[label=(\arabic*.), nosep,wide, labelindent=0pt]
        \item \textbf{non-trivial}: $\hankel_\statep \neq \zero$, for all $\statep \in \prefixes$; 
         \item  \defn{closed}: for every $\statep \in \prefixes$ and $\sym \in \alphabet$ such that $\hankel_{\prefix \sym}\neq \zero$, there exists $\stateq  \in \prefixes $ such that %
         $\statep \,\sym \sim_\hankel \stateq $---in particular, $ \prefixesplus/\sim_\hankel=\prefixes/\sim_\hankel$; and
         \item \defn{consistent}: for every $\statep, \stateq  \in \prefixes$ %
         \begin{equation}
             \statep \sim_\hankel \stateq  \Rightarrow \statep\, \sym  \sim_\hankel \stateq  \,\sym, \quad \forall \sym \in \alphabet .
         \end{equation}
     \end{enumerate}
 \end{definition}
 We define the dimension of an empirical Hankel 
 system to be $\dim(\emphankel)\defeq |\prefixes/\sim_\hankel|$.
 Given an empirical Hankel matrix $\hankel\colon \prefixesplus \times \suffixesplus \rightarrow \weightset$ and a  weighted language $\lang$, we say that $\lang$ 
 \defn{contains} $\hankel$ if $\hankel(\prefix, \suffix)= \lang(\prefix \suffix)$ for every $\prefix \in \prefixesplus$ and $\suffix \in \suffixesplus$. 
 Likewise, we say that a WFSA $\automaton$ contains  $\hankel$ if the language $\lang_{\automaton}$ contains $\hankel$.\looseness=-1

 We define a \defn{partial order} on the set of empirical Hankel systems as follows: given two empirical Hankel systems $\emphankel_1 = (\prefixes_1, \suffixes_1, \hankel_1)$ and $\emphankel_2 = (\prefixes_2, \suffixes_2, \hankel_2)$, we define $\emphankel_1 \preceq \emphankel_2$, if
 $\prefixes_1 \subset \prefixes_2$, $\suffixes_1\subset \suffixes_2$ and $\hankel_1(\prefix,\suffix)=\hankel_2(\prefix,\suffix)$ for any $(\prefix,\suffix) \in \prefixes_1 \circ \alphabet^{\leq 1} \times \suffixes_1 \circ \alphabet^{\leq 1}$. 
 \subsection{Minimal Hankel Automaton}
 \begin{restatable}[$\naivehankelautomaton$ is transition-regular]{theorem}{transitionRegularAutomaton}\label{thm:AhTransReg}
     Let $\emphankel = (\prefixes, \suffixes, \hankel)$ be an empirical Hankel system. 
     The equivalence relation $\sim_\hankel$ on $\naivehankelautomaton$ is transition-regular (see \cref{transreg}), which means that for every $\statep \in \prefixes$ and every $\sym \in \alphabet$:
     \begin{enumerate}[nosep,wide, labelindent=0pt]
     \item There exists $\stater \in \prefixes$ such that $\edge{\statep}{\sym}{\weight_1}{\stater} \in \transf_\hankel$ for some $\weight_1 \in \weightset\setminus \{\zero\}$.
     \item If $\stateq \in \prefixes$ is another prefix with $\statep \sim_\hankel \stateq$, then:
     \noindent (a) for all $\stater\in \prefixes$:
    $$ \edge{\statep}{\sym}{\weight_1}{\stater} \in \transf_\hankel \iff \edge{\stateq}{\sym}{\weight_2}{\stater} \in \transf_\hankel
    $$
    and: (i) $\stater  \sim_\hankel \statep \sym  \sim_\hankel \stateq \sym$, (ii) $\weight_1=\weight_2$. 
    
    \noindent (b) %
     $\initf_\hankel(\statep) = \initf_\hankel(\stateq)$ and $ \finalf_\hankel(\statep) = \finalf_\hankel(\stateq).$
\end{enumerate}
 \end{restatable}
 \begin{proof}Fix a prefix $\statep \in \prefixes$ and a symbol $\sym \in \alphabet$. 
 
 1.     Since $\hankel$ is {closed}, there exists $\stater \in \prefixes$ and $k \in \weightset\setminus\{\zero\}$ such that $\hankel_\stater = k \otimes \hankel_{\statep \sym}$, which implies, by definition of $\naivehankelautomaton$, that  $\edge{\statep}{\sym}{\weight_1}{\stater} \in \transf_\hankel$ for some $\weight_1 \in \weightset$. 
        Since $\emphankel$ is non-trivial, $d_\hankel(\stater) = k \otimes d_\hankel(\statep \sym) \neq \zero$, and so $w_1=\frac{d_\hankel(\statep \sym)}{d_\hankel(\statep)} \neq \zero$. 

2. (a.i) 
            By definition of $\naivehankelautomaton$, if $\edge{{\statep}}{\sym}{\weight_1}{   {\stater}   }$ is in $\transf_\hankel$, then $\statep \,\sym \sim_\hankel \stater$ for some $\stater \in \prefixes$. 
            Now, let $\stateq \in \prefixes$ such that $\statep \sim_\hankel \stateq$. 
            By {consistency} of $\hankel$, we have $\stateq \,\sym \sim_\hankel \statep \,\sym \sim_\hankel \stater$ and so 
            $\edge{{\stateq}}{\sym}{\weight_2}{   {\stater}   } \in \transf_\hankel$ for some $\weight_2 \in \weightset$. 
            The reverse follows similarly.
     
 2. (a.ii) Let us show $\weight_2 = \weight_1$. 
            By assumption, we know that there exists $k \in \weightset\setminus\{\zero\}$ such that $
                \hankel_{\statep }(\suffix) = k\otimes \hankel_{\stateq}(\suffix)$ $\forall \suffix \in \suffixesplus$.
            Hence ${d_\hankel(\statep\sym)} =  \oplus_{\suffix \in \suffixes} \hankel_\prefix(\sym\suffix) = k \otimes {d_\hankel(\stateq\sym)}$ 
            and ${d_\hankel(\statep)} =  \oplus_{\suffix \in \suffixes} \hankel_\prefix(\suffix) = k  \otimes {d_\hankel(\stateq)}$.
            Accordingly,  ${d_\hankel(\statep)} \neq \zero \iff {d_\hankel(\stateq)} \neq \zero$ and in which case 
            \begin{equation}
                \weight_1= \frac{{d_\hankel(\statep\sym)}}{{d_\hankel(\statep)}}=\frac{{d_\hankel(\stateq\sym)}}{{d_\hankel(\stateq)}}=\weight_2
            \end{equation}         
2. (b) If $\statep \sim_\hankel\stateq$, then we have
        \begin{equation}
            \finalf_\hankel(\statep)= \frac{\hankel_\statep(\eps)}{ d_\hankel(\statep)}=\frac{\cancel{k}\otimes\hankel_\stateq(\eps)}{\cancel{k}\otimes d_\hankel(\stateq)}=\finalf_\hankel(\stateq)
        \end{equation}
        The computation for $\initf_\hankel(\statep)$ follows similarly.   
 \end{proof}

 \begin{restatable}[The empirical Hankel Automaton $\emphankelautomaton$]{theorem}{empHankelDFSA} \label{thm:empirical-Hankel-theorem}
    Let $\emphankel = (\prefixes, \suffixes, \hankel)$ be an empirical Hankel system and let %
    $\emphankelautomaton$ 
    be the quotient of $\naivehankelautomaton$ modulo the transition-regular equivalence relation $\sim_\hankel$ as defined in \cref{quotientaut}. 
    Then:
    \begin{enumerate}[label=\textit{(\arabic*)}, nosep,wide, labelindent=0pt]
        \item The weighted automaton $\emphankelautomaton$ is  trimmed and deterministic.
        \item $\AutLang{\emphankelautomaton}(\prefix\suffix) = \hankel(\prefix, \suffix)$ for all $\prefix \in \prefixes$ and $\suffix \in \suffixes$, 
        meaning that $\emphankelautomaton$ contains $\hankel$.
        \end{enumerate}
\end{restatable}
\begin{proof}
\textit{(1)} Since the automaton is built on an empirical Hankel system, by definition,  every $\statep$ is the prefix of a string $\strx=\prefix \suffix$, such that $\hankel(\prefix,\suffix)\neq \zero$ for at least one $\suffix \in \suffixes$, hence $\statep$ is accessible and coaccessible. 
This shows that $\naivehankelautomaton$ is trimmed, and so is $\emphankelautomaton$. 
Determinism and \textit{(2)} follow from \cref{quotientautlang}.
\end{proof}

\begin{restatable}[Minimality of $\emphankelautomaton$]{theorem}{minimalityTheorem} \label{thm:empirical-hankel-automaton-minimal} \
\begin{enumerate}[label=\textit{(\roman*)}, nosep,wide, labelindent=0pt]
    \item For any $\statep,\stateq\in \prefixes$, we have 
    \begin{equation}
    \statep \sim_\hankel \stateq \iff \statep \sim_{\lang_{\emphankelautomaton}} \stateq
    \end{equation}
    \item $\prefixes/\sim_\hankel = \supp{\AutLang{\emphankelautomaton}}/\sim_{\lang_\emphankelautomaton}$.
    \item Any automaton that contains  $\emphankel$ must have at least $|\prefixes / \sim_{\hankel}|=|\states_{\emphankelautomaton}|$ states. 
    \item Let $\automaton'$ be a WDFSA that contains $\emphankel$. 
    Then, $\lang_{\automaton'}(\strx)=\lang_{\naivehankelautomaton}(\strx), \quad \forall \strx \in \supp{\lang_{\naivehankelautomaton}}$. 
    If $\automaton'$ is not equivalent to $\naivehankelautomaton$, then $$|\states_{\automaton'}|\ge |\prefixes/\sim_\hankel|+1.$$ 
    \item In particular, $\emphankelautomaton$ is minimal.
\end{enumerate}
\end{restatable}
 
\begin{proof} \
\begin{enumerate}[label=\textit{(\roman*)}, nosep,wide, labelindent=0pt]
    \item ($\Leftarrow$). 
 If $\inv{\statep} \lang_{\emphankelautomaton} \equiv_{{\kleene{\alphabet}}} \inv{\stateq} \lang_{\emphankelautomaton} $, then restricting the two maps to $\suffixesplus$ shows that we also have $\hankel_{\statep}  \equiv_{\suffixesplus} \hankel_{\stateq} $. 
 ($\Rightarrow$). Clearly, $\statep \sim_\hankel \stateq$ implies $[\statep] = [\stateq]$.

    \item From \textit{(i)}, we have that the restriction
 \begin{align}
     \{\inv{\statep} &\lang_\emphankelautomaton \colon \statep \in \prefixes\}    \twoheadrightarrow \\
     &\{\inv{\statep}\hankel= \inv{\statep}{\lang_\emphankelautomaton}|_{\suffixes}\colon \statep \in \prefixes \} \nonumber
 \end{align} 
 provides a natural surjection.
 Let $\strx \in \supp{\lang_{\emphankelautomaton}}$ and $\apath_\strx$ its path in $\emphankelautomaton$. 
 Let $\statep_\strx = \lastq(\apath_\strx) \in \prefixes$ be the final state of $\apath_\strx$ in $\automaton_\hankel$.  
 Thus, by definition, we have $\inv{\statep_\strx}\lang_\emphankelautomaton \equiv_{\kleene{\alphabet}}\inv{\strx}\lang_\emphankelautomaton $. 
 Accordingly, the natural projection map $\prefixes \to \supp{\lang_{\emphankelautomaton}}/\equiv_{\kleene{\alphabet}}$ is surjective, 
 and hence we have a bijection
 $\prefixes/\sim_{\lang_\emphankelautomaton}  \simeq \supp{\lang_{\emphankelautomaton}}/\equiv_{\kleene{\alphabet}}$. 
 Since in \textit{(i)} we showed $\prefixes/\sim_\hankel \simeq \prefixes/\sim_{\lang_\emphankelautomaton}$,  we conclude
 \begin{equation}
 \prefixes/\sim_\hankel = \supp{\lang_{\emphankelautomaton}}/\sim_{\lang_\emphankelautomaton}
 \end{equation}
\item Let $\automaton'$ be any WDFSA that contains $\emphankel$. 
Clearly, we have
$\statep \sim_{\lang_{\automaton'}} \stateq \Rightarrow \statep \sim_{\hankel} \stateq$, 
hence we have a surjective map $\prefixes/\sim_{\AutLang{\automaton'}} \twoheadrightarrow  \prefixes/\sim_{\hankel}$, which shows
\begin{equation}
|\states_{\automaton'}|\ge |\prefixes/\sim_{\lang_{\automaton'}}|   \ge |\prefixes/\sim_{\hankel}|=|\states_{\emphankelautomaton}|
\end{equation}
\item Consider the sub-WDFSA $\automaton'_\prefixes \subset \automaton'$ with states
\begin{equation}
\states'_\prefixes= \{\stateq \in \states_{\automaton'}\mid \stateq= \lastq{\apath_\statep} \text{ for } \statep \in \prefixes\}
\end{equation}
Clearly $\automaton'_\prefixes$ contains $\emphankel$. 
In addition, because $\emphankel$ is closed and consistent, $\automaton'_\prefixes$ and $\automaton_\hankel$ have the same transitions (not necessarily same weights). 
We hence have  $\lang_{\automaton_\prefixes'}=\lang_{\automaton_\hankel}$. 
In other words, $\lang_{\automaton'}(\strx)=\lang_{\automaton_\hankel}(\strx)$ for all $\strx \in \supp{\emphankelautomaton}$. 
Accordingly, if $\lang_{\automaton'}\neq\lang_{\automaton_\hankel}$, then there exists $\strx\in \supp{\lang_{\automaton'}}\setminus \supp{\lang_{\automaton_\hankel}} $ and 
$[\strx] \in (\supp{\lang_{\automaton'}} / \sim_{\lang_{\automaton'}} )\setminus (\supp{\lang_{\automaton_\hankel}} /\sim_{\lang_{\automaton'}})$. %
Thus:
\begin{align*}
    |\states_{\automaton'}|&\ge |\supp{\lang_{\automaton'}}/\sim_{\lang_{\automaton'}}|\\
    &\ge |\supp{\lang_{\automaton_\hankel}}/\sim_{\lang_{\automaton'}}|+1\\
    &\ge|\supp{\lang_{\automaton_\hankel}}/\sim_{\lang_{\automaton_\hankel}}|+1\\
    &= |\prefixes/\sim_{\hankel}|+1
\end{align*}
\item If in particular, $\automaton'$ is any WDFSA equivalent to $\emphankelautomaton$, then by \textit{(iii)} $|\states_{\automaton'}|\ge |\states_{\emphankelautomaton}|$. \qedhere
\end{enumerate}
\end{proof}
\begin{corollary}[Termination]\label{cor:termination}
    Let $\emphankel$ be an empirical Hankel system and $\automaton'$ any automaton that contains it. 
    If $|\states_{\automaton'}|=|\prefixes/\sim_\hankel|$, then $\lang_{\automaton_\hankel}=\lang_{\automaton'}$. 
\end{corollary}

\section{A Weighted \texorpdfstring{\lstar}{L-star} Algorithm}

Like \citet{ANGLUIN198787}, we assume we have access to an \defn{oracle} that answers the following queries about a deterministic regular language $\unknownLang\colon\!\kleene{\alphabet}\!\rightarrow\!\weightset$:
\begin{enumerate}[label=\textit{(\arabic*)},nosep,wide, labelindent=0pt]
    \item \defn{Membership query}: What is the weight $\unknownLang\left(\prefix\right)$ of the string $\prefix \in \kleene{\alphabet}$?
    \item \defn{Equivalence query}: Does a hypothesis automaton $\emphankelautomaton$
          generate $\unknownLang$?
          If it does \emph{not}, the oracle provides a \defn{counterexample}, which is a string $\tt$ such that $L_{\emphankelautomaton}(\tt)\neq \unknownLang (\tt)$.
\end{enumerate}
At a high level, the algorithm iteratively constructs empirical Hankel systems of increasing dimensions that capture observed patterns of the target language $\unknownLang$. 
Once sufficient observations are accumulated, the automaton derived from these Hankel systems will generate exactly $\unknownLang$.

\subsection{The Learning Algorithm}
Our weighted \lstar algorithm, with its main loop detailed in \cref{alg:weighted-learner}, employs the subroutines outlined in \cref{alg:closed-consistent}.
\begin{algorithm}[t]
    \begin{algorithmic}[1]
        \Func{\lstar($\oracle$)}
        \label{alg:line A1-initialize}
        \While{$\textbf{true}$} \label{line:outer-loop}
        \While{$\textbf{true}$ } \label{alg:inner-loop}
        \If{$\hankel$\text{ is \textbf{not} consistent}} \label{alg:line-A1-consistency}
        \State $\consistent(\oracle, \hankel)$ \label{alg:line-A1-consistent}
        \ElsIf{$\hankel$ \text{ is \textbf{not} closed}}
        \State $\closed(\oracle, \hankel )$ \label{alg:line-A1-closed}
        \Else $\text{ : }$ $\textbf{break}$ \label{alg:line-A1-break}
        \EndIf
        \EndWhile \label{alg:line A1-closeness and consistency last line}
        \label{alg:line A1-hypothesis automaton}
        \State $\emphankel \gets \removeNull(\hankel )$
        \label{alg:line generate hankel system}
        \State $\emphankelautomaton \gets \makeAutomaton(\emphankel)$ \label{alg:line make-automaton}
        \If{$\equivalence(\oracle, \emphankelautomaton)$} $\Return\; \emphankelautomaton$ \label{alg:line A1-equivalence query}
        \Else $\text{ : }$
        \State $\prefix \gets \textsc{Counterexample}(\oracle, \emphankelautomaton)$\label{alg:line A1-counterexample}
        \For{$t=1$ \textbf{to} \label{alg:line A1-counterexample-1}$|\prefix|+1$}\label{alg:line A1-add counterexample prefixes}
        \State $\prefixes \gets \prefixes \cup \{ \prefix_{< t} \} $ \label{alg:line A1-counterexample-2}
        \EndFor \label{alg:line A1-counterexample-3}
        \State $\complete(\oracle, \hankel)$ \label{alg:line A1-fill table with membership after counterexample is added}        \EndIf
        \EndWhile
        \EndFunc
    \end{algorithmic}
    \caption{The Weighted \lstar algorithm. Initially, the empirical Hankel matrix $\hankel$ is set to the zero matrix and the sets $\prefixes,\suffixes$ to $\{ \eps \}$.}
    \label{alg:weighted-learner}
\end{algorithm}
\looseness=-1
\paragraph{Initialization.} 
$\prefixes$ and $\suffixes$ are initialized as $\{\eps\}$ and the $\hankel$ to the zero matrix. 

\paragraph{Handling inconsistencies.} 
\consistent in \cref{alg:line-A1-closed} of \cref{alg:weighted-learner} looks for rows $\prefix, \prefix' \in \prefixes$ that make $\hankel$ non-consistent, i.e., $\hankel_{\prefix \sym} \not{\equiv}_{\suffixesplus} \hankel_{\prefix ' \sym}$: It normalizes a row $\hankel_{\prefix \sym}$ as $\frac{\hankel_{\prefix \sym}}{d_{\hankel}(\prefix \sym )}$  
(\cref{alg:closed-consistent}, \cref{alg:line A2-consistent: normalize and compare}), which allows testing homothetic equivalence with equality.\footnote{When the entire row $\hankel_{\prefix \sym}$ is zero, we do not normalize; this is omitted in the pseudocode for brevity.}
For every $\suffix \in \suffixes$ that makes $\hankel$ inconsistent, $ \sym \suffix$ is added to $\suffixes$. 
This results in the new equivalence classes $[\prefix]$ and $[\prefix']$ because $\hankel_{\prefix}$ and $\hankel_{\prefix'}$ do not match anymore on the column indexed by $\sym \suffix$. 
See \cref{lem:ec-increase-1} for more details.\looseness=-1
\paragraph{Closing $\hankel$.} 
\closed (\cref{alg:weighted-learner}, \cref{alg:line-A1-closed}; \cref{alg:closed-consistent}) adds to $\prefixes$ the missing prefixes required to make $\hankel$ closed.
This results in the new equivalence class $[\prefix \sym ]$.
See \cref{lem:ec-increase-1} for more details.

\paragraph{Filling out $\hankel$.}
\complete fills the empty entries of $\hankel$ by asking membership queries to the oracle.\looseness=-1

Handling inconsistencies, closing $\hankel$, and filling $\hankel$ is carried out by the inner while loop (Lines \ref{alg:inner-loop} to \ref{alg:line-A1-break}) of \cref{alg:weighted-learner} until $\hankel$ is both closed and consistent.

\paragraph{Generating $\emphankelautomaton$.}
When $\hankel$ is closed and consistent, \cref{alg:weighted-learner} first removes $\zero$-rows from the matrix to obtain an empirical Hankel system, then it generates the empirical Hankel automaton $\emphankelautomaton$, and lastly submits an equivalence query to the oracle (\cref{alg:line A1-equivalence query}).
If the oracle answers positively, \cref{alg:weighted-learner} halts and returns $\emphankelautomaton$. 
Otherwise, the oracle provides a counterexample, which is added to $\prefixes$ along with its prefixes. 
$\hankel$ is then updated through membership queries (\cref{alg:line A1-counterexample,alg:line A1-counterexample-1,alg:line A1-counterexample-2,alg:line A1-counterexample-3,alg:line A1-fill table with membership after counterexample is added}).
The algorithm continues until $\hankel$ is closed and consistent again.

\begin{algorithm}[t]
    \begin{algorithmic}[1]
        \Func{$\consistent(\oracle, \hankel)$} \label{alg:line A2-make consistent}
        \For{$\langle \prefix, \prefix' \rangle \in \prefixes \times \prefixes$}
        \If{$\hankel_{\prefix}\equiv_{\suffixesplus} \hankel_{\prefix'}$}
        \For{$ \langle \sym, \suffix \rangle \in  \alphabet \times \suffixes$}
        \vspace{1 mm}
        \If{$\frac{\hankel_{\prefix \sym}(\suffix)}{d_{\hankel}(\!\prefix \sym \!)} \neq  \frac{ \hankel_{\prefix' \sym}(\suffix)}{d_{\hankel}(\prefix' \sym )}$}  \label{alg:line A2-consistent: normalize and compare}
        \vspace{1 mm}
        \State $\suffixes \gets \suffixes \cup \{ \sym \suffix  \}$ \label{alg: line A2-add suffix}
        \EndIf
        \EndFor
        \EndIf
        \EndFor
        \State $\complete( \oracle, \hankel)$
        \EndFunc
        \vspace{-3.66 mm}
        \Statex
        \Func{$\closed(\oracle, \hankel)$} \label{alg:line A2-make closed}
        \For{$\langle \prefix, \sym \rangle \in \prefixes \times \alphabet$}
        \If{$\not \exists \prefix' \in \prefixes$ s.t. $\hankel_{\prefix \sym} \equiv_{\suffixesplus}\hankel_{\prefix'}$}
        \State $\prefixes \gets \prefixes \cup \{ \prefix \sym \}$
        \EndIf
        \EndFor
        \State $\complete(\oracle, \hankel)$
        \EndFunc
        \vspace{-3.66 mm}
        \Statex
        \Func{$\complete(\oracle, \hankel)$} \label{alg:line A2-fill table}
        \For{$\prefix \in  \prefixesplus $}
        \For{$\suffix \in \suffixesplus $}
        \State $\hankel(\prefix, \suffix) \gets \membership(\oracle, \prefix \suffix )$
        \EndFor
        \EndFor
        \EndFunc
    \end{algorithmic}
    \caption{Subroutines of \cref{alg:weighted-learner}.}
    \label{alg:closed-consistent}
\end{algorithm}

\begin{restatable}{theorem}{halting} \label{thm:halting} 
    Let $\weightset$ be a semifield and $\alphabet$ an alphabet.
    Let $\oracle$ be an oracle for a deterministic regular language $\unknownLang\colon \kleene{\alphabet} \rightarrow \weightset$, whose minimal WDFSA has $N$ states.
    Then, \cref{alg:weighted-learner} returns a minimal WDFSA generating $\unknownLang$ in time $\bigo{N^5 M^2 \abs{\alphabet}^2}$, where $M$ is the length of the longest counterexample that $\oracle$ can provide.
\end{restatable}
\begin{proof}
    See \cref{sec:proof-halting}.
\end{proof}

\section{Conclusion}
We introduce a weighted \lstar algorithm, an oracle-based algorithm for learning weighted regular languages, building upon the paradigm pioneered by \citet{ANGLUIN198787}.
While similar methods have been proposed before, our method is novel in that it learns an \emph{exact} deterministic WFSA, akin to the original \citeposs{ANGLUIN198787} unweighted version.

\section*{Limitations}
One of the limitations of weighted \lstar is that it requires an oracle capable of answering membership and equivalence queries. However, in 
the case we want to use \lstar to study a language model, this is the ideal setting, as we can use the language model itself as the oracle \citep{pmlr-v80-weiss18a,okudono2019weighted,weiss-weighted-l}. Another 
limitation to the applications of our work is that not every language model is efficiently 
representable as a finite-state machine. For 
instance, \citet{merrill-2019-sequential} shows that LSTMs are strictly more powerful than FSAs. 
Therefore, in practice, one may have to use a 
simplified abstraction of the model one aims to learn \citep{weiss-weighted-l}, inevitably 
reducing the model's expressivity.

\section*{Acknowledgements}
Ryan Cotterell acknowledges support from the Swiss National Science
Foundation (SNSF) as part of the ``The Nuts and Bolts of Language Models'' project.
Anej Svete is supported by the ETH AI Center Doctoral Fellowship.

\bibliography{custom}
\bibliographystyle{acl_natbib}
\appendix

\onecolumn

\section{Transition-regular Equivalence Relations on Automata}

\begin{definition}\label{transreg}
Let $\automaton = (\alphabet, \states, \transf, \initf, \finalf)$ be a WFSA. 
An equivalence relation $\sim$ on $\states$ is \defn{transition-regular} if, for any states $\statep, \stateq \in \states$, whenever $\statep \sim \stateq$, we have:
\begin{itemize}
    \item \textbf{Outgoing Transition Consistency}: For every symbol $\sym \in \alphabet$, if there exists a state $\stater \in \states$ such that $\edge{\statep}{a}{\weight_1}{\stater} \in \transf $ 
    with weight $\weight_1\neq \zero$, then \begin{itemize}
        \item there must exist a transition $\edge{\stateq}{\sym}{\weight_2}{\stater} \in \transf$ with $\weight_2 = \weight_1$.
        \item for any other state $\stater' \in \states$ such that $\edge{\statep}{a}{\weight_1'}{\stater'} \in \transf $  we must have $\stater \sim \stater'$. 
    \end{itemize}
   
    \item \textbf{Initial and Final Weight Consistency}: The initial and final weights of $\statep$ and $\stateq$ are identical:
    \begin{subequations}
    \begin{align}
        \initf(\statep) &= \initf(\stateq) \\
        \finalf(\statep) &= \finalf(\stateq).
    \end{align}
    \end{subequations}
\end{itemize}
\end{definition}

\begin{definition}\label{quotientaut}
Let $\automaton = (\alphabet, \states, \transf, \initf, \finalf)$ be a WFSA. 
Given a transition-regular equivalence relation $\sim$ on $\states$, we define the \textbf{quotient automaton} $\quotient{\automaton} = (\quotient{\states}, \alphabet, \quotient{\transf}, \quotient{\initf}, \quotient{\finalf})$ as follows:

\begin{itemize}
    \item \textbf{States}: The state set of the quotient automaton is $
    \quotient{\states} = \states / \sim.
    $
    \item \textbf{Transitions}: Define the transition set $\quotient{\transf}$ as follows. For each equivalence class $[\statep] \in \quotient{\states}$ and each symbol $\sym \in \alphabet$, if there exists a state $\stateq \in \states$ such that $\edge{\statep}{\sym}{\weight}{\stateq} \in \transf$, then there is a corresponding transition
    $$
    \edge{[\statep]}{\sym}{\weight}{[\stateq]} \in \quotient{\transf},
    $$
    where $[\stateq]$ denotes the equivalence class of $\stateq$.
    
    \item \textbf{Initial and Final Weights}: Define the initial and final weight functions for each equivalence class $[\statep] \in \quotient{\states}$ as:
    \begin{align*}
        \quotient{\initf}([\statep]) &= \initf(\statep), \\
        \quotient{\finalf}([\statep]) &= \finalf(\statep).
    \end{align*}
\end{itemize}
\end{definition}
\begin{lemma}\label{quotientautlang}
    The quotient automaton $\quotient{\automaton}$ is deterministic and it generates the following weighted language:
\begin{equation}
 \AutLang{\automaton}(\prefix)   = |\Paths{\automaton}{\prefix}|  \AutLang{\quotient{\automaton}}(\prefix) , \qquad \forall\prefix \in \kleene{\alphabet}.
\end{equation}
\end{lemma}
\begin{proof}
    The proof is straightforward. 
    
\end{proof}

\section{Proof of \texorpdfstring{\cref{thm:halting}}{Thm. 3}} \label{sec:proof-halting}
\halting*

\subsection{Terminination}
We begin by stating the following lemma 

\begin{lemma}[Number of equivalence classes increases] \label{lem:ec-increase-1}
\label{lemma: increase}
   When \cref{alg:weighted-learner}
   \begin{enumerate}[label=(\arabic*),nosep,wide, labelindent=0pt]
       \item adds a suffix to $\suffixes$ because $\hankel$ is not consistent,
       \item adds a prefix to $\prefixes$ because the table is not closed,
       \item adds a prefix to $\prefixes$ because the oracle replied with a counterexample,
   \end{enumerate}
   the number of equivalence classes $\prefixes/\sim_{\hankel}$ increases.
\end{lemma}
\begin{proof} \
    \begin{enumerate}[label=\textit{(\arabic*)}, nosep,wide, labelindent=0pt]
        \item If the empirical Hankel matrix is not consistent, \consistent (\cref{alg:weighted-learner}, \cref{alg:line-A1-consistent}) finds two prefixes $\prefix, \prefix' \in \prefixes$ such that $ \hankel_{\prefix} \equiv \hankel_{\prefix'}$ but $\hankel_{\prefix \sym} \not \equiv \hankel_{\prefix' \sym }$. 
        Then it searches for a tuple $(\sym,\suffix)$, $\sym \in \alphabet, \suffix \in \suffixesplus$ that makes the relation
        $\hankel_{\prefix \sym}\equiv \hankel_{\prefix' \sym}$ false, and adds $\sym \suffix$ to $\suffixes$. 
        After adding $\sym \suffix$ to $\suffixes$, we have that $\hankel_{\prefix} \not \equiv \hankel_{\prefix'}$, and therefore the equivalence class $[\prefix]$ is divided in two different ones.
        \item If $\hankel$ is not closed, \closed (\cref{alg:weighted-learner}, \cref{alg:line-A1-closed}) finds $\prefix \in \prefixes$ and $\sym \in \alphabet$ such that $\hankel_{\prefix  \sym} \nequiv \hankel_{\prefix'}$ for every $\prefix' \in \prefixes$ and adds $\prefix \sym$ to $\prefixes$. Since there was no $\prefix'$ in $\prefixes$ such that $\prefix' \sim_{\hankel} \prefix \sym $, it follows that a new equivalence class $[\prefix \sym]$ is added to $\prefixes / \sim_{\hankel}$. 
        \item When the Oracle replies negatively to the equivalence query, the counterexample $\tt$, together with its prefixes, is added to $\prefixes$. We show that even in this case, $\dim(\emphankel)$ increases. Let $\emphankel$ and $\emphankel^{\tt}$ denote the empirical Hankel system before and after adding $\tt$, and let $\emphankelautomaton$ and $\widetilde{\automaton}_{\hankel^{\tt}}$ denote the corresponding empirical Hankel automaton in each case. We note that both $\emphankelautomaton$ and $\widetilde{\automaton}_{\hankel^{\tt}}$ contain $\emphankel$ and therefore by \cref{thm:empirical-hankel-automaton-minimal}---since the automata are not equivalent---$\emphankelautomaton^{\tt}$ must have at least one more state than $\emphankelautomaton$. By construction of the empirical Hankel automaton, we know that this implies that $\dim(\emphankel)$ increases.\qedhere
        \end{enumerate}
\end{proof}

Let $(\emphankel_k= (\prefixes_k,\suffixes_k,\hankel_k))_{k\ge 0}$ be the sequence of empirical Hankel systems constructed at each iteration of the main loop of \cref{alg:weighted-learner}. 
By \cref{lemma: increase}, this sequence is increasing; that is, {$\emphankel_k \preceq \emphankel_{k+1}$ for all $k\ge 0$}. 
Let $\unknownAut$ denote any minimal automaton for the target language $\unknownLang$. 
On the one hand, we know that the automaton $\unknownAut$ contains $\emphankel_k$ for all $k\ge 0$. 
On the other hand, by \cref{lemma: increase}, there exists $n \in \N$ such that $\dim(\emphankel_n) =|\prefixes_n/\sim_{\automaton_{\hankel_n}}|= |\states_{\unknownAut}|$. 
Therefore, by applying \cref{cor:termination}, we conclude that $\unknownLang=\lang_{\automaton_{\hankel_n}}$. 

Consequently, after a finite number of iterations, the oracle will respond positively to the equivalence query, causing the algorithm to halt. 
Furthermore, we observe that the inner loop of \cref{alg:weighted-learner} executes at most  $|\states_\unknownAut|$ times, as $\dim(\emphankel)$ increases at every step by \cref{lemma: increase}, and at each iteration, $\unknownAut$ continues to contain $\hankel$.

\subsection{Run-Time}
First, we note that since $\unknownLang$ always contains $\hankel$, by \cref{lemma: increase} any of the following events can only occur at most $N$ times in total, where $N$ is the number of states of the minimal automaton accepting $\unknownLang$: \emph{i)} we add a prefix because the table is not closed, \emph{ii)} we add a suffix because the table is not consistent, \emph{iii)} we add a counterexample because the oracle replies negatively to the equivalence query.
Then, let us give a bound on the size of the prefix and suffix sets:
\begin{itemize}[nosep]
    \item $|\prefixes| \in \bigo{NM}$, in fact initially $\prefixes = \{ \varepsilon \}$, and $\prefixes$ can be incremented at most $N$ times because the matrix is not consistent and at most $NM$ times because the oracle replies with a counterexample, where $M$ is the length of the longest counterexample.
    \item $|\suffixes| \in \bigo{N}$, in fact initially $\suffixes = \{ \varepsilon \}$, and $\suffixes$ can be incremented at most $N$ times.
\end{itemize}

Next, we shall analyze the runtime of the operations executed during the main loop of \cref{alg:weighted-learner}.

\begin{itemize}[nosep]
    \item $\consistent \in \bigo{ \abs{\prefixes}^2\abs{\suffixes}^2\abs{\alphabet}^2}$
    \item $\closed \in \bigo{ \abs{\prefixes}^2 \abs{\suffixes} \abs{\alphabet}^2}$
    \item $ \complete \in \bigo{ \abs{\prefixes} \abs{\suffixes} \abs{\alphabet}^2}$
    \item $\makeAutomaton \in \bigo{\abs{\prefixes}+\abs{\prefixes} \abs{\alphabet}}$, 
\end{itemize}

In the analysis above we used the fact that the map $d_{\hankel}: \prefixesplus \rightarrow \weightset$ is fixed for every empirical Hankel matrix and can be precomputed in time $\bigo{|\prefixes| |\suffixes| |\alphabet|^2}$ and reused multiple times.

We note that each of these operations can be executed at most $N$ times before the algorithm halts, and therefore---by substituting in the bounds for $\prefixes$ and $\suffixes$---we compute the total runtime of \cref{alg:weighted-learner} as:
\begin{subequations}
\begin{align}
    &\bigo{N\left( N^4 M^2 \abs{\alphabet}^2 + N^3 M^2 \abs{\alphabet}^2 + N^2 M \abs{\alphabet}^2 + N^2 M \abs{\alphabet} + MN + NM\abs{\alphabet} \right)} \\
    & =\bigo{N^5 M^2 \abs{\alphabet}^2}
\end{align}
\end{subequations}

We note an important distinction between our weighted version of $\lstar$ and \citeposs{ANGLUIN198787}. In fact, in the weighted case, we need to make sure that the empirical Hankel matrix has a column for every element $\suffixesplus$ and not only for $\suffixes$. This is a fundamental step to enforce that the relation $\sim_{\hankel}$ is transition regular (\cref{thm:AhTransReg}), and it is related to the fact that in the weighted case, we don't seek language equality, but rather equality modulo multiplication by a constant $k$.

\end{document}

%% file: macros.tex
\newcommand{\mymacro}[1]{{#1}}

\newcommand{\defn}[1]{\textbf{#1}}

\newcommand{\ifcondition}{\textbf{ if }}
\newcommand{\otherwisecondition}{\textbf{ otherwise }}

\newcommand{\N}{\mymacro{\mathbb{N}}}

\def\mH{{{\mymacro{ \mathbf{H}}}}}

\DeclareMathSymbol{\mlq}{\mathord}{operators}{``} %
\DeclareMathSymbol{\mrq}{\mathord}{operators}{`'} %

\newcommand{\abs}[1]{\lvert#1\rvert}

\NewDocumentCommand\otable{g}{{\mymacro{ \IfNoValueTF{#1}{ \langle \prefixes, \suffixes, \hankel \rangle}{\langle \prefixes{#1}, \suffixes{#1}, \hankel{#1} \rangle } }}}

\newcommand{\hankel}{{\mymacro{\mH}}}

\newcommand{\emphankel}{{\mymacro{\overline{\mH}}}}

\newcommand{\myalgoname}[1]{\mymacro{\textsc{#1}}}
\newcommand{\membership}{\myalgoname{Membership}\xspace}
\newcommand{\equivalence}{{\myalgoname{Equivalent}}\xspace}
\newcommand{\complete}{{\myalgoname{Complete}}\xspace}
\newcommand{\closed}{{\myalgoname{MakeClosed}}\xspace}
\newcommand{\consistent}{{\myalgoname{MakeConsistent}}\xspace}
\newcommand{\makeAutomaton}{{\myalgoname{MakeAutomaton}}\xspace}
\newcommand{\removeNull}{{\myalgoname{RemoveNullRows}}\xspace}

\newcommand{\prefixes}{{\mymacro{{\mathrm{P}}}}}%
\newcommand{\suffixes}{{\mymacro{{\mathrm{S}}}}} %

\newcommand{\prefixesplus}{\prefixes\circ{\alphabet^{\leq 1}}}
\newcommand{\suffixesplus}{{\alphabet^{\leq 1}} \circ  \suffixes}

\makeatletter
\newcommand{\notpropto}{\mymacro{{\propto\kern-1\@ptsize pt \diagup}}}

\newcommand{\prefix}{{\mymacro{\boldsymbol{p}}}}
\newcommand{\suffix}{{\mymacro{\boldsymbol{s}}}}
\newcommand{\sym}{{\mymacro{ a}}}
\newcommand{\eps}{{\mymacro{\varepsilon}}}
\newcommand{\str}{\yy}
\newcommand{\strx}{\xx}

\newcommand{\strz}{\zz}

\renewcommand{\tt}{{ \mymacro{ \boldsymbol{t}}}}
\newcommand{\oracle}{{\mymacro{ \mathbb{O}}}}

\newcommand{\unknownLang}{\ensuremath{\mymacro{\textcolor{ETHBronze}{L^{\bigstar}}}}}
\newcommand{\unknownAut}{\ensuremath{\mymacro{\textcolor{ETHBronze}{\automaton^{\bigstar}}}}}

\newcommand{\edge}[4]{{ \mymacro{ #1 \xrightarrow[]{ #2 / #3 } #4}}}

\newcommand{\one}{{ \mymacro{ \boldsymbol{1}}}}
\newcommand{\zero}{{ \mymacro{{ \boldsymbol{0}} }}}
\newcommand{\weight}{{\mymacro{ w}}}

\newcommand{\lang}{\mymacro{\mathbf{L}}}

\newcommand{\weightset}{{\mymacro{ \mathbb{K}}}}

\newcommand{\semiringtuple}{{\mymacro{ \langle  \weightset, \oplus, \otimes,\zero, \one \rangle}}}

\newcommand{\alphabet}{{ \mymacro{ \Sigma}}}
\newcommand{\inv}[1]{{ \mymacro{ #1^{-1}} }}
\newcommand{\quotient}[1]{\widetilde{#1}}

\newcommand{\xx}{{ \mymacro{ \mathbf{x}}}}
\newcommand{\yy}{{ \mymacro{ \mathbf{y}}}}
\newcommand{\zz}{{\mymacro{ \mathbf{z}}}}
\newcommand{\kleene}[1]{#1^{*}}
\newcommand{\statep}{{\mymacro{\boldsymbol{p}}}}
\newcommand{\stateq}{{\mymacro{\boldsymbol{q}}}}
\newcommand{\stater}{{\mymacro{\boldsymbol{r}}}}
\newcommand{\states}{{\mymacro{Q}}}

\newcommand{\automaton}{{\mymacro{ \mathcal{A}}}}
\newcommand{\apath}{{\mymacro{ \boldsymbol{\pi}}}}
\newcommand{\pathweight}[1]{{ \mymacro{ \boldsymbol{w}\left( #1 \right)}}}
\newcommand{\pathyield}[1]{{ \mymacro{ \sigma\left( #1\right) }}}
\NewDocumentCommand\Paths{gg}{%
    \IfNoValueTF{#2}{%
        {\mymacro{ \ensuremath{ \Pi_{#1} }  }}%
    }{%
        {\mymacro{ \ensuremath{ \Pi_{#1}\left(#2\right) }}}%
    }
}
\newcommand{\AutLang}[1]{{\mymacro{ \ensuremath{\lang_{#1}}}}}

\NewDocumentCommand\firstq{g}{%
    \IfNoValueTF{#1}{%
        {\mymacro{ \ensuremath{ \textit{i} }  }}%
    }{%
        {\mymacro{ \ensuremath{ \textit{i}\left(#1\right) }}}%
    }
}

\NewDocumentCommand\Languages{g}{\mymacro{ \IfNoValueTF{#1}{{\mathcal{L}}}{\mathcal{L}{\left(#1\right)} }}}
\NewDocumentCommand\qlang{gg}{\mymacro{ \IfNoValueTF{#2}{{\ell_{#1}}}{\ell_{#1}\left( #2 \right)}}}
\NewDocumentCommand\rightlang{ggg}{\mymacro{ \IfNoValueTF{#3}{{\inv{#2}#1}}{\inv{#2}#1\left(#3\right)}}}

\newcommand{\transf}{{\mymacro{ \ensuremath{\delta}}}}
\newcommand{\initf}{{\mymacro{ \ensuremath{\lambda}}}}
\newcommand{\finalf}{{\mymacro{ \ensuremath{\rho}}}}

\NewDocumentCommand\wfsatuple{g}{%
    \IfNoValueTF{#1}{%
        {\mymacro{ \ensuremath{ \left( \alphabet, \states, \transf, \initf, \finalf \right) }  }}%
    }{%
        {\mymacro{ \ensuremath{ \left( \alphabet, \states{#1}, \transf{#1}, \initf{#1}, \finalf{#1} \right) }  }}%
    }
}
\newcommand{\qinit}[0]{\mymacro{\ensuremath{q_{I}}}}

\NewDocumentCommand\lastq{g}{%
    \IfNoValueTF{#1}{%
        {\mymacro{ \ensuremath{ \textit{f} }  }}%
    }{%
        {\mymacro{ \ensuremath{ \textit{f}\left(#1\right) }}}%
    }
}

\NewDocumentCommand\state{g}{\mymacro{ \IfNoValueTF{#1}{{\bf q}}{ {\bf q}{#1} }}}

\newcommand{\defeq}{\overset{\textnormal{\tiny def}}{=}}

\newcommand{\emphankelautomaton}{\mymacro{\widetilde{\automaton}_{\hankel}}}
\newcommand{\naivehankelautomaton}{\mymacro{\automaton_{\hankel}}}

\newcommand{\bigo}[1]{{\mymacro{ \mathcal{O}\left( #1\right)}}}

\newcommand{\lstar}{\mymacro{\ensuremath{\mathbf{L^*}}}\xspace}

\newcommand{\supp}[1]{\mymacro{\ensuremath{\text{supp}(#1)\xspace}}}

%% file: main.bbl
\begin{thebibliography}{19}
\providecommand{\natexlab}[1]{#1}

\bibitem[{Allauzen and Mohri(2003)}]{twins}
Cyril Allauzen and Mehryar Mohri. 2003.
\newblock \href {https://jalc.de/issues/2003/issue_8_2/abs-117.pdf} {Efficient algorithms for testing the twins property}.
\newblock \emph{Journal of Automata, Languages and Combinatorics}, 8(2):117–144.

\bibitem[{Angluin(1987)}]{ANGLUIN198787}
Dana Angluin. 1987.
\newblock \href {https://doi.org/10.1016/0890-5401(87)90052-6} {Learning regular sets from queries and counterexamples}.
\newblock \emph{Information and Computation}, 75(2):87--106.

\bibitem[{Balle et~al.(2014)Balle, Carreras, Luque, and Quattoni}]{Balle-Quattoni-2014}
Borja Balle, Xavier Carreras, Franco~M. Luque, and Ariadna Quattoni. 2014.
\newblock \href {https://doi.org/10.1007/s10994-013-5416-x} {Spectral learning of weighted automata}.
\newblock \emph{Machine Learning}, 96(1):33--63.

\bibitem[{Balle and Mohri(2012)}]{Balle-Mohri-Spectral}
Borja Balle and Mehryar Mohri. 2012.
\newblock \href {https://proceedings.neurips.cc/paper_files/paper/2012/file/700fdb2ba62d4554dc268c65add4b16e-Paper.pdf} {Spectral learning of general weighted automata via constrained matrix completion}.
\newblock In \emph{Advances in Neural Information Processing Systems}, volume~25.

\bibitem[{Balle and Mohri(2015)}]{Balle2015}
Borja Balle and Mehryar Mohri. 2015.
\newblock \href {https://doi.org/10.1007/978-3-319-23021-4_1} {Learning weighted automata}.
\newblock In \emph{Algebraic Informatics}, pages 1--21. Springer International Publishing.

\bibitem[{Beimel et~al.(2000)Beimel, Bergadano, Bshouty, Kushilevitz, and Varricchio}]{beimel}
Amos Beimel, Francesco Bergadano, Nader~H. Bshouty, Eyal Kushilevitz, and Stefano Varricchio. 2000.
\newblock \href {https://doi.org/10.1145/337244.337257} {Learning functions represented as multiplicity automata}.
\newblock \emph{J. ACM}, 47(3):506–530.

\bibitem[{Bergadano and Varricchio(1996)}]{Learning-behaviours-bergadano}
Francesco Bergadano and Stefano Varricchio. 1996.
\newblock \href {https://doi.org/10.1137/S009753979326091X} {Learning behaviors of automata from multiplicity and equivalence queries}.
\newblock \emph{SIAM Journal on Computing}, 25(6):1268--1280.

\bibitem[{Daviaud and Johnson(2024)}]{daviaud2024feasability}
Laure Daviaud and Marianne Johnson. 2024.
\newblock \href {https://arxiv.org/abs/2309.07806} {Feasability of learning weighted automata on a semiring}.
\newblock \emph{Preprint}, arXiv:2309.07806.

\bibitem[{Gold(1978)}]{GOLD}
E~Mark Gold. 1978.
\newblock \href {https://doi.org/10.1016/S0019-9958(78)90562-4} {Complexity of automaton identification from given data}.
\newblock \emph{Information and Control}, 37(3):302--320.

\bibitem[{Goodman(1996)}]{goodman-1996-parsing}
Joshua Goodman. 1996.
\newblock \href {https://doi.org/10.3115/981863.981887} {Parsing algorithms and metrics}.
\newblock In \emph{34th Annual Meeting of the Association for Computational Linguistics}, pages 177--183, Santa Cruz, California, USA. Association for Computational Linguistics.

\bibitem[{Hopcroft and Ullman(1979)}]{Hopcroft-Automata-Theory}
John~E. Hopcroft and Jeffrey~D. Ullman. 1979.
\newblock \href {https://en.wikipedia.org/wiki/Introduction_to_Automata_Theory,_Languages,_and_Computation} {\emph{Introduction to Automata Theory, Languages, and Computation}}.
\newblock Addison-Wesley Publishing Company, Inc.

\bibitem[{Jumelet and Zuidema(2023)}]{jumelet-zuidema-2023-transparency}
Jaap Jumelet and Willem Zuidema. 2023.
\newblock \href {https://doi.org/10.18653/v1/2023.findings-emnlp.288} {Transparency at the source: {Evaluating} and interpreting language models with access to the true distribution}.
\newblock In \emph{Findings of the Association for Computational Linguistics: EMNLP 2023}, pages 4354--4369, Singapore. Association for Computational Linguistics.

\bibitem[{Merrill(2019)}]{merrill-2019-sequential}
William Merrill. 2019.
\newblock \href {https://doi.org/10.18653/v1/W19-3901} {Sequential neural networks as automata}.
\newblock In \emph{Proceedings of the Workshop on Deep Learning and Formal Languages: Building Bridges}, pages 1--13, Florence. Association for Computational Linguistics.

\bibitem[{Mohri(1997)}]{mohri-1997-finite}
Mehryar Mohri. 1997.
\newblock \href {https://aclanthology.org/J97-2003} {Finite-state transducers in language and speech processing}.
\newblock \emph{Computational Linguistics}, 23(2):269--311.

\bibitem[{Mohri(2009)}]{Mohri2009}
Mehryar Mohri. 2009.
\newblock \href {https://doi.org/10.1007/978-3-642-01492-5_6} {\emph{Weighted Automata Algorithms}}, pages 213--254.
\newblock Springer Berlin Heidelberg, Berlin, Heidelberg.

\bibitem[{Nowak et~al.(2024)Nowak, Svete, Butoi, and Cotterell}]{nowak-etal-2024-computational}
Franz Nowak, Anej Svete, Alexandra Butoi, and Ryan Cotterell. 2024.
\newblock \href {https://doi.org/10.18653/v1/2024.acl-long.676} {On the representational capacity of neural language models with chain-of-thought reasoning}.
\newblock In \emph{Proceedings of the 62nd Annual Meeting of the Association for Computational Linguistics (Volume 1: Long Papers)}, pages 12510--12548, Bangkok, Thailand. Association for Computational Linguistics.

\bibitem[{Okudono et~al.(2019)Okudono, Waga, Sekiyama, and Hasuo}]{okudono2019weighted}
Takamasa Okudono, Masaki Waga, Taro Sekiyama, and Ichiro Hasuo. 2019.
\newblock \href {https://arxiv.org/abs/1904.02931} {Weighted automata extraction from recurrent neural networks via regression on state spaces}.
\newblock \emph{Preprint}, arXiv:1904.02931.

\bibitem[{Weiss et~al.(2018)Weiss, Goldberg, and Yahav}]{pmlr-v80-weiss18a}
Gail Weiss, Yoav Goldberg, and Eran Yahav. 2018.
\newblock \href {https://proceedings.mlr.press/v80/weiss18a.html} {Extracting automata from recurrent neural networks using queries and counterexamples}.
\newblock In \emph{Proceedings of the 35th International Conference on Machine Learning}, volume~80 of \emph{Proceedings of Machine Learning Research}, pages 5247--5256. PMLR.

\bibitem[{Weiss et~al.(2019)Weiss, Goldberg, and Yahav}]{weiss-weighted-l}
Gail Weiss, Yoav Goldberg, and Eran Yahav. 2019.
\newblock \href {https://proceedings.neurips.cc/paper_files/paper/2019/file/d3f93e7766e8e1b7ef66dfdd9a8be93b-Paper.pdf} {Learning deterministic weighted automata with queries and counterexamples}.
\newblock In \emph{Advances in Neural Information Processing Systems}, volume~32.

\end{thebibliography}
